\documentclass{amsart}
\usepackage{amsmath,amsthm,amssymb}
\usepackage{booktabs}
\usepackage{multirow}

\usepackage{graphicx}
\usepackage{enumitem}
\usepackage{hyperref}
\usepackage{url}

\newcommand{\pr}{\mathrm{pr}}
\newcommand{\id}{\mathrm{id}}
\newcommand{\IN}{\mathbb{N}}

\newcommand{\Ssolvable}[1]{{#1}\text{-}\mathrm{solv}}
\newcommand{\preserves}{\mathrel{\triangleright}}

\DeclareMathOperator{\Pol}{Pol}
\DeclareMathOperator{\Inv}{Inv}

\newcommand{\clAll}{\mathsf{\Omega}}
\newcommand{\clL}{\mathsf{L}}
\newcommand{\clOmegaOne}{\mathsf{\Omega(1)}}
\newcommand{\clI}{\mathsf{I}}
\newcommand{\clC}{\mathsf{C}}
\newcommand{\clIneg}{\mathsf{N}}
\newcommand{\clgen}[1]{\langle{#1}\rangle}
\newcommand{\nset}[1]{\ensuremath{[{#1}]}}
\newcommand{\vect}[1]{\ensuremath{\mathbf{#1}}}

\theoremstyle{plain}%
\newtheorem{theorem}{Theorem}%  meant for continuous numbers
\newtheorem{proposition}[theorem]{Proposition}% 
\newtheorem{lemma}[theorem]{Lemma}
\newtheorem{fact}[theorem]{Fact}%

\theoremstyle{definition}%
\newtheorem{example}[theorem]{Example}%
\newtheorem{definition}[theorem]{Definition}%

\begin{document}

\title[Galois theory for analogical classifiers]{Galois theory for analogical classifiers}

\author{Miguel Couceiro}
\address[M. Couceiro]%
   {Universit\'e de Lorraine, CNRS, LORIA \\
    F-54000 Nancy \\
    France}

\author{Erkko Lehtonen}
\address[E. Lehtonen]%
   {Centro de Matem\'atica e Aplica\c{c}\~oes \\
    Faculdade de Ci\^encias e Tecnologia \\
    Universidade Nova de Lisboa \\
    Quinta da Torre \\
    2829-516 Caparica \\
    Portugal}

\begin{abstract}Analogical proportions are 4-ary relations that read ``A is to B as C is to D''. Recent works have highlighted the fact that such relations can support a specific form of inference, called analogical inference. This inference mechanism was empirically proved to be efficient in several reasoning  and classification tasks. 
In the latter case, it relies on the notion of analogy preservation.  

In this paper, we explore this relation between formal models of analogy and the corresponding classes of analogy preserving functions, and we establish a Galois theory of analogical classifiers. We illustrate the usefulness of this Galois framework over Boolean domains, and  we explicitly determine the closed sets of analogical classifiers, i.e., classifiers that are compatible with the analogical inference, for each pair of Boolean analogies. 
\end{abstract}




\maketitle

\section{Introduction and motivation}\label{sec1}

Analogical reasoning (AR) is a remarkable capability of human thought that exploits parallels between situations of different nature to infer plausible conclusions, by relying simultaneously on similarities and dissimilarities. 
Machine learning (ML) and artificial intelligence (AI) have tried to develop AR, mostly based on cognitive considerations, and to integrate it in a variety of ML tasks, such as natural language processing (NLP), preference learning and recommendation \cite{alsaidia,FahandarH18,FahandarH21,Mitchell21}. 
Also, analogical extrapolation (inference) can solve difficult reasoning tasks such as {\it scholastic aptitude tests} and {\it visual question answering} \cite{SadeghiZF15,PeyreSLS19}. Inference based on AR can also support dataset augmentation (analogical extension and extrapolation) for model learning, especially in environments with few labeled examples \cite{CouceiroHPR17}. Furthermore, AR can also be performed at a meta level for transfer learning \cite{CornuejolsMO20,alsaidi21b} where the idea is to take advantage of what has been learned on a source domain in order to improve the learning process in a target domain related to the source domain.
Moreover, analogy making can provide useful explanations that rely on the parallel example-counterexample \cite{Hullermeier20} and guide counterfactual generation \cite{KeaneS20}.

However, early works lacked theoretical and formalizational support. The situation started to change about a decade ago when researchers adopted the view of analogical proportions as statements of the form ``$a$ relates to $b$ as $c$ relates to $d$'', usually denoted $a:b::c:d$. 
%Formally, they can be expressed by $E(R(A,B),R(C,D))$, where $E$ denotes the ``as'' part and where $R$ can be instantiated in many different ways depending on the underlying representation of $A,B,C$ and $D$, or on the domain of interpretation.
Such proportions are at the root of the analogical inference mechanism, and several formalisms to study this mechanism have been proposed, which follow different axiomatic and logical approaches \cite{Lepage01,MicletBD08}. 
% After: For instance, Lepage \cite{Lepage01} considered 
%Describe quickly different sets of postulates unified into a single 
For instance, Lepage \cite{Lepage03} introduces the following 4 postulates in the linguistic
context as a guideline for formal models of analogical proportions: {\it symmetry} (if $a:b::c:d$, then $c:d::a:b$),
{\it central permutation}  (if $a:b::c:d$, then $a:c::b:d$),  {\it strong inner reflexivity} (if $a:a::c:d$, then $d = c$), and
{\it strong reflexivity} (if $a:b::a:d$, then $d = b$).
Such postulates appear reasonable in the word domain, but they can be criticized in other application domains. For instance, in a setting where two distinct conceptual spaces are involved, as in $\mathit{wine} : \mathit{French} :: \mathit{beer} : \mathit{Belgian}$ where two different spaces ``drinks'' and ``nationality'' are considered, the central permutation is not tolerable.
% Recently, Antic \cite{Antic2020} proposed an algebraic framework of analogies that is  naturally embedded into first-order logic via model-theoretic types. It provides a unifying setting where the different axiomatic approaches in the literature and respective domains of interpretation  can be considered. Note that different axiomatic approaches entail different dataset augmentation procedures.

%There are two basic tasks associated with AR. The first is \emph{analogy making} that corresponds to the task of detecting and deciding whether a quadruple $(a, b, c, d)$ corresponds to a valid analogical proportion. 
%The second is \emph{analogy solving} that refers to the task of finding or extrapolating, for a given triple $a,b,c$ the value $x$ such that $a : b :: c : x$ is a valid analogy.
%This task is typically addressed in the literature by retrieval and adaptation, i.e., defining an  $x$ from a pool of retrieved candidate solutions to be suitably adapted. 
A key task associated with AR is \emph{analogy solving}, i.e. finding or extrapolating, for a given triple $a,b,c$ a value $x$ such that $a : b :: c : x$ is a valid analogy. 
In fact, this task can be seen as central in case-based reasoning (CBR)~\cite{richter2016case}. Given a set $\mathcal{P}$ of problems, a set $\mathcal{S}$ of solutions and a set $\mathit{CB}$ of problem-solution tuples $(x,y)\in \mathcal{P} \times \mathcal{S}$, called \emph{cases}, the CBR task consists in finding a solution $y_t \in \mathcal{S}$ to a given target problem $x_{t} \in \mathcal{P}$. 
The CBR methodology splits this problem into several steps, the two most important being
(1) \emph{retrieval:} select $k$ ``relevant" source cases in the case-base~$\mathit{CB}$ according to some criteria related to the target problem; and
(2) \emph{adaptation}: reuse the $k$ retrieved cases for proposing a solution to the target problem. 
The adaptation step obviously depends on the number of retrieved source cases.
For $k=1$, the desired solution $y_t$ corresponds to the solution of the analogical equation $x:y::x^t:y^t$. For higher values of $k$, different models of analogy on $\mathcal{P}$ and $\mathcal{S}$ could be taken into account. For instance, when $k=3$, the retrieval task consists in finding a triple of cases $(x^1, y^1)$,  $(x^2, y^2)$ and $(x^3, y^3)$, such that $x^1:x^2::x^3:x^t$ is valid and such that $y^1:y^2::y^3:y$ is solvable in $y$~\cite{LieberNP21}. In this setting the desired $y^t$ would then be one of such solutions.

The latter idea was extended to analogy based classification \cite{BounhasPR17a} where objects are viewed as attribute tuples (instances) $\mathbf{x}=(x_1,\ldots,x_n)$. 
Similarly, if $\mathbf{a},\mathbf{b},\mathbf{c}$ are in analogical proportion for most of their attributes, and class labels are known for $\mathbf{a},\mathbf{b},\mathbf{c}$ but unknown for $\mathbf{d}$, then one may infer the label for $\mathbf{d}$ as a solution of an analogical proportion equation. 
All these applications rely on the same idea: if four instances $\mathbf{a},\mathbf{b},\mathbf{c},\mathbf{d}$ are in analogical proportion for most of the attributes describing them, then it may still be the case for the other attributes $f(\mathbf{a}), f(\mathbf{b}), f(\mathbf{c}), f(\mathbf{d})$ (for some function $f$).
This principle is called \emph{analogical inference principle} (AIP).

Theoretically, it is quite challenging to find and characterize situations where AIP can be soundly applied. 
A first step toward explaining the analogical mechanism consists in characterizing the set of functions~$f$ for which AIP is sound ({\it i.e.}, no error occurs) no matter which triplets of examples are used. 
%In case of Boolean attributes, a first step for explaining the analogical mechanism was to characterize the set of functions for which AIP is sound ({\it i.e.}, no error occurs) no matter which triplets of examples are used. 
In case of Boolean attributes and for the model of proportional analogy, it was shown in~\cite{CouceiroHPR17} that these so-called ``analogy-preserving'' (AP) functions coincide exactly with the set of affine Boolean functions. 
Moreover, it was also shown that, when the function is not affine, the prediction accuracy remains high if the function is close to being affine~\cite{CouceiroHPR18}. In fact, it was shown  that if $f$ is $\varepsilon$-approximately affine (i.e., $f$ is at Hamming distance at most $\varepsilon$ from the class of affine functions), then
the $f$'s average error rate is at most $ 4 \varepsilon$.
These results were recently extended in \cite{CouceiroLMPR20} to nominal (finite) underlying sets when taking the {\it minimal model} of analogy, i.e., only patterns of the form $x:x::y:y$ and $x:y::x:y$, in both the domain and codomain of classifiers. 

Intuitively, this class will change when adopting different models of analogy. In this paper, we investigate the relation between formal models of analogy and the corresponding class of AP functions. In this paper we establish a formal correspondence between them describe a {\it Galois theory of analogical classifiers}.
Moreover, we revisit Anti\'c's analogical formalism in the Boolean case and, for each pair of formal models of analogy, we explicitly determine the corresponding closed sets of analogical classifiers.

\section{Galois theories for functions}\label{sec2}

Let $A$ and $B$ be nonempty sets.
A \emph{function of several arguments} from $A$ to $B$ is a mapping $f \colon A^n \to B$ for some natural number $n$ called the \emph{arity} of $f$.
Denote by $\mathcal{F}_{AB}^{(n)}$ the set of all $n$-ary functions of several arguments from $A$ to $B$, and let $\mathcal{F}_{AB} := \bigcup_{n \in \IN} \mathcal{F}_{AB}^{(n)}$.
In the case when $A = B$ we speak of \emph{operations} on $A$, and we use the notation $\mathcal{O}_A^{(n)} := \mathcal{F}_{AA}^{(n)}$ and $\mathcal{O}_A := \mathcal{F}_{AA}$.
For any set $C \subseteq \mathcal{F}_{AB}$, the \emph{$n$-ary part} of $C$ is $C^{(n)} := C \cap \mathcal{F}_{AB}^{(n)}$.

If $f \in \mathcal{F}_{BC}^{(n)}$ and $g_1, \dots, g_n \in \mathcal{F}_{AB}^{(m)}$, then the \emph{composition} $f(g_1, \dots, g_n)$ belongs to $\mathcal{F}_{AC}^{(m)}$ and is defined by the rule
\[
f(g_1, \dots, g_n)(\mathbf{a}) := f(g_1(\mathbf{a}), \dots, g_n(\mathbf{a}))
\quad\text{for all $\mathbf{a} \in A^m$.}
\]
The $i$-th $n$-ary \emph{projection} $\pr_i^{(n)} \in \mathcal{O}_A^{(n)}$ is defined by $\pr_i^{(n)}(a_1, \dots, a_n) := a_i$ for all $a_1, \dots, a_n \in A$.
We denote by $\mathcal{J}_A$ the set of all projections on $A$.

The notion of functional composition can be extended to sets of functions as follows.
Let $C \subseteq \mathcal{F}_{BC}$ and $K \subseteq \mathcal{F}_{AB}$.
The \emph{composition} of $C$ with $K$ is the set
\begin{multline*}
CK := \\
\{ h \in \mathcal{F}_{AC} \mid \exists m, n \in \IN, \, f \in C^{(n)}, \, g_1, \dots, g_n \in K^{(m)}, \, h = f(g_1, \dots, g_n) \}.
\end{multline*}

A \emph{clone} on $A$ is a set $C \subseteq \mathcal{O}_A$ that is closed under composition and contains all projections, in symbols, $C C \subseteq C$ and $\mathcal{J}_A \subseteq C$.
For $F \subseteq \mathcal{O}_A$, we denote by $\clgen{F}$ the clone generated by $F$, i.e., the smallest clone on $A$ containing $F$.

Let $f \in \mathcal{F}_{AB}^{(n)}$ and $g \in \mathcal{F}_{AB}^{(m)}$.
We say that $f$ is a \emph{minor} of $g$, and we write $f \leq g$, if $f \in \{g\} \mathcal{J}_A$, or, equivalently, there exists a $\sigma \colon \{1, \dots, m\} \to \{1, \dots, n\}$ such that
\[
f(a_1, \dots, a_n) = g(a_{\sigma(1)}, \dots, a_{\sigma(m)})
\quad\text{for all $a_1, \dots, a_n \in A$.}
\]
The minor relation $\leq$ is a quasi-order (a reflexive and transitive relation) on $\mathcal{F}_{AB}$.
Downsets of $(\mathcal{F}_{AB}, {\leq})$ are called \emph{minor-closed} classes or \emph{minions.}
Equivalently, a set $C \subseteq \mathcal{F}_{AB}$ is a minion if $C \mathcal{J}_A \subseteq C$.

A set $\mathcal{C} \subseteq \mathcal{F}_{AB}$ is \emph{$m$-locally closed} if for all $f \in \mathcal{F}_{AB}$ (say $f$ is $n$-ary), it holds that $f \in C$ whenever for every finite subset $S \subseteq A^n$ of size at most $m$, there exists a $g \in \mathcal{C}$ such that $f \vert_S = g \vert_S$.
A set $\mathcal{C}$ is said to be \emph{locally closed} if it is  $m$-locally closed for every positive integer $m$.

Subsets of $A^m$ are called $m$-ary \emph{relations} on $A$.
Denote by $\mathcal{R}_A^{(m)}$ the set of all $m$-ary relations on $A$, and let $\mathcal{R}_A := \bigcup_{m \in \IN} \mathcal{R}_A^{(m)}$.
Let $f \in \mathcal{O}_A^{(n)}$ and $R \in \mathcal{R}_A^{(m)}$.
We say that the function $f$ \emph{preserves} the relation $R$ (or $f$ is a \emph{polymorphism} of $R$, or $R$ is an \emph{invariant} of $f$), and we write $f \preserves R$, if for all $\mathbf{a}_1, \dots, \mathbf{a}_n \in R$, we have $f(\mathbf{a}_1, \dots, \mathbf{a}_n) \in R$.
Here $f(\mathbf{a}_1, \dots, \mathbf{a}_n)$ means the componentwise application of $f$ to the tuples, i.e., if $\mathbf{a}_i = (a_{i1}, \dots, a_{im})$ for $i \in \{1, \dots, n\}$, then
\[
f(\mathbf{a}_1, \dots, \mathbf{a}_n) :=
(f(a_{11}, \dots, a_{n1}), \dots, f(a_{1m}, \dots, a_{nm})).
\]
%\pa{I think it should be 1m and nm instead of m1 and mn in the equation.}
The preservation relation $\preserves$ induces a Galois connection between the sets $\mathcal{O}_A$ and $\mathcal{R}_A$ of operations and relations on $A$.
Its polarities are the maps $\Pol \colon \mathcal{P}(\mathcal{R}_A) \to \mathcal{P}(\mathcal{O}_A)$ and $\Inv \colon \mathcal{P}(\mathcal{O}_A) \to \mathcal{P}(\mathcal{R}_A)$ given by the following rules: for all $\mathcal{R} \subseteq \mathcal{R}_A$ and $\mathcal{F} \subseteq \mathcal{O}_A$,
\begin{align*}
\Pol \mathcal{R} &:= \{ f \in \mathcal{O}_A \mid \forall R \in \mathcal{R} \colon f \preserves R \}, \\
\Inv \mathcal{F} &:= \{ R \in \mathcal{R}_A \mid \forall f \in \mathcal{F} \colon f \preserves R \}.
\end{align*}
Under this Galois connection, the closed sets of operations are precisely the locally closed clones.
The closed sets of relations, known as \emph{relational clones,}
%were also described in terms of closure under certain operations on relations.
are precisely the locally closed sets of relations that contain the empty relation and the diagonal relations and are closed under formation of primitively positively definable relations.
This was first shown for finite base sets by Bodnarchuk, Kaluzhnin, Kotov, Romov~\cite{BKKR1,BKKR2} and Geiger~\cite{Geiger} and later extended for arbitrary sets by Szab\'o~\cite{Szabo} and P\"oschel~\cite{Poschel-PolInv}.

The preservation relation can be adapted for functions of several arguments from $A$ to $B$; we now need to consider pairs of relations.
Let 
$$\mathcal{R}_{AB}^{(m)} := \mathcal{R}_A^{(m)} \times \mathcal{R}_B^{(m)}\quad\text{ and}\quad \mathcal{R}_{AB} := \bigcup_{m \in \IN} \mathcal{R}_{AB}^{(m)}$$ 
be the set of all ($m$-ary) \emph{relational constraints} from $A$ to $B$.

Let $f \in \mathcal{F}_{AB}^{(n)}$ and $(R,S) \in \mathcal{R}_{AB}^{(m)}$.
We say that $f$ \emph{preserves} $(R,S)$ (or $f$ is a \emph{polymorphism} of $(R,S)$, or $(R,S)$ is an \emph{invariant} of $f$), and we write $f \preserves (R,S)$, if for all $\mathbf{a}_1, \dots, \mathbf{a}_n \in R$, we have $f(\mathbf{a}_1, \dots, \mathbf{a}_n) \in S$.
As before, the preservation relation $\preserves$ induces a Galois connection between the sets $\mathcal{F}_{AB}$ and $\mathcal{R}_{AB}$ of functions and relational constraints from $A$ to $B$.
Its polarities are the maps $\Pol \colon \mathcal{P}(\mathcal{R}_{AB}) \to \mathcal{P}(\mathcal{F}_{AB})$ and $\Inv \colon \mathcal{P}(\mathcal{F}_{AB}) \to \mathcal{P}(\mathcal{R}_{AB})$ given by the following rules: for all $\mathcal{Q} \subseteq \mathcal{R}_{AB}$ and $\mathcal{F} \subseteq \mathcal{F}_{AB}$,
\begin{align*}
\Pol \mathcal{Q} &:= \{ f \in \mathcal{F}_{AB} \mid \forall (R,S) \in \mathcal{Q} \colon f \preserves (R,S) \}, \\
\Inv \mathcal{F} &:= \{ (R,S) \in \mathcal{R}_{AB} \mid \forall f \in \mathcal{F} \colon f \preserves (R,S) \}.
\end{align*}
The sets $\Pol \mathcal{Q}$ and $\Inv \mathcal{F}$ are said to be \emph{defined} by $\mathcal{Q}$ and $\mathcal{F}$, respectively.
Sets of functions of the form $\Pol \mathcal{Q}$ for some $\mathcal{Q} \subseteq \mathcal{R}_{AB}$ and sets of relational constraints of the form $\Inv \mathcal{F}$ for some $\mathcal{F} \subseteq \mathcal{F}_{AB}$ are said to be \emph{definable} by relational constraints and functions, respectively.

%Under this Galois connection, the closed sets of operations are precisely the locally closed minions.
%This was first shown for finite base sets by Pippenger~\cite{Pippenger} and later extended for arbitrary sets by Couceiro and Foldes~\cite{CouFol}.

%The description of the closure system associated to this Galois connection for arbitrary, possibly infinite, sets $A$ and $B$ was given in \cite{CouFol} with the aid of the notion of {\it local closure}.

The closed sets of functions under this Galois connection were described for finite base sets by Pippenger~\cite{Pippenger} and later for arbitrary sets by Couceiro and Foldes~\cite{CouFol}.
This result was refined by Couceiro~\cite{Couceiro05} for sets of functions definable by relations of restricted arity.

\begin{theorem}[{\cite{CouFol,Couceiro05}}] 
Let $A$ and $B$ be arbitrary nonempty sets, and let $C \subseteq \mathcal{F}_{AB}$.
\begin{enumerate}[label=\upshape{(\roman*)}]
    \item $C$ is definable by constraints if and only if $C$ is a locally closed minion.
    \item $C$ is definable by constraints of arity $m$ if and only if $C$ is an $m$-locally closed minion.
\end{enumerate}
\end{theorem}

The closed sets of relational constraints were described in terms of closure conditions that parallel those for relational clones.

The description of the dual objects of constraints on possibly infinite sets $A$ and $B$ was also provided in \cite{CouFol} and inspired by those given by Geiger \cite{Geiger,Szabo,Poschel-PolInv,Pippenger} and given in terms of positive primitive first-order relational definitions applied simultaneously on antecedents and consequents. Sets of constraints that are closed under such formation schemes are said to be {\it closed under conjunctive minors}. Moreover, every function satisfies the empty $(\varnothing,\varnothing)$ and the equality $(=_A,=_B)$ constraints, and if a function $f$ satisfies a constraint $(R,S)$, then $f$ also satisfies its {\it relaxations} $(R',S')$ such that $R'\subseteq R$ and $S'\supseteq S$. 
 
 As for functions, in the infinite case, we also need to consider a ``local closure'' condition to describe the dual closed sets of relational constraints on $A$ and $B$. A set $\mathcal{Q}$ of constraints on $A$ and $B$ is $n$-\emph{locally closed} if it contains every relaxation of its members whose antecedent has size at most $n$, and it is \emph{locally closed} if it is $n$-locally closed for every  positive integer $n$.

\begin{theorem}[{\cite{CouFol,Couceiro05}}]  
For arbitrary nonempty sets $A$ and $B$, and let $\mathcal{Q}\subseteq \mathcal{R}_{AB}$ be a set of relational constraints on $A$ and $B$.
\begin{enumerate}[label=\upshape{(\roman*)}]
    \item $\mathcal{Q}$ is definable by some set $\mathcal{C}\subseteq\mathcal{F}_{AB}$ if and only if it is locally closed, contains the binary equality and the empty constraints, and it is closed under relaxations and conjunctive minors.
    \item $\mathcal{Q}$ is definable by some set $\mathcal{C}\subseteq\mathcal{F}_{AB}^{(n)}$ of $n$-ary functions if and only it is $n$-locally closed, contains the binary equality and the empty constraints, and it is closed under relaxations and conjunctive minors.
\end{enumerate}
 
\end{theorem}

Let $K \subseteq \mathcal{F}_{AB}$ and let $C_1$ and $C_2$ be clones on $A$ and $B$.
We say that $K$ is \emph{stable under right composition with $C_1$} if $K C_1 \subseteq K$, and we say that $K$ is \emph{stable under left composition with $C_2$} if $C_2 K \subseteq K$.
We say that $K$ is \emph{$(C_1,C_2)$-stable} or a \emph{$(C_1,C_2)$-clonoid,} if $K C_1 \subseteq K$ and $C_2 K \subseteq K$.

Motivated by earlier results on linear definability of equational classes of Boolean functions \cite{CouFol04} which were described in terms of stability under compositions with the clone of constant preserving affine functions,  Couceiro and Foldes \cite{CouFol09} introduced a Galois framework for describing sets of functions $\mathcal{F}\subseteq \mathcal{F}_{AB}$ stable under right and left compositions with clones $C_1$ on $A$ and $C_2$ on $B$, respectively.

For that they restricted the defining dual objects to relational constraints $(R,S)$ where $R$ and $S$ invariant under  $C_1$ and $C_2$, respectively, i.e., $R \in \Inv C_1$ and $S \in \Inv C_2$.
These were referred to as \emph{$(C_1,C_2)$-constraints}.
We denote by $\mathcal{R}_{AB}^{(C_1,C_2)}$ the set of all $(C_1,C_2)$-constraints.

\begin{theorem}[{\cite{CouFol09}}]\label{C12clonoids} 
Let  $A$ and $B$ be arbitrary nonempty sets, and let $C_1$ and $C_2$ clones on $A$ and $B$, respectively.
A set $\mathcal{C}\subseteq\mathcal{F}_{AB}$ 
is definable by some set of $(C_1,C_2)$-constraints if and only if 
 $\mathcal{C}$ is locally closed and stable under right and left composition with $C_1$ and $C_2$, respectively, {\it i.e.}, it is a locally closed $(C_1,C_2)$-clonoid. 
\end{theorem}

Dually, a set $\mathcal{Q}$ of $(C_1,C_2)$-constraints is definable by a set $\mathcal{C}\subseteq \mathcal{F}_{AB}$ if $\mathcal{Q} = \Inv \mathcal{C} \cap \mathcal{R}_{AB}^{(C_1,C_2)}$.

To describe the dual closed sets of $(C_1,C_2)$-constraints, Couceiro {\it et al.}\ \cite{CouFol09} observed that conjunctive minors of $(C_1,C_2)$-constraints are themselves $(C_1,C_2)$-constraints. However, this is not the case for relaxations. They thus proposed the following variants of local closure and of constraint relaxations.

A set $\mathcal{Q}_0$ of  $(C_1,C_2)$-constraints is said to be \emph{$(C_1,C_2)$-locally closed} if the set $\mathcal{Q}$ of all relaxations of the various constraints in $\mathcal{Q}_0$ is locally closed.
A relaxation $(R_0,S_0)$ of a relational constraint $(R,S)$ is said to be a \emph{$(C_1,C_2)$-relaxation} if $(R_0,S_0)$ is a $(C_1,C_2)$-constraint.

\begin{theorem}[{\cite{CouFol09}}] \label{thm:4}
Let  $A$ and $B$ be arbitrary nonempty sets, and let $C_1$ and $C_2$ clones on $A$ and $B$, respectively.
A set $\mathcal{Q}$ of $(C_1,C_2)$-constraints 
is definable by some set 
 $\mathcal{C}\subseteq \mathcal{F}_{AB}$ if and only if it is $(C_1,C_2)$-locally closed and contains the binary equality constraint,
 the empty constraint, and it is closed under $(C_1,C_2)$-relaxations and conjunctive minors.
\end{theorem}

In this paper, we will focus on relational constraints whose antecedent and consequent are derived from analogies, and that we will refer to as \emph{analogical constraints}. We will denote the  set of all analogical constraints from $A$ to $B$ by $\mathcal{A}_{AB}$.

\section{Formal models of analogy}\label{sec3}

%{\flushleft\bf To be edited...}

In this section we  briefly survey different axiomatic settings to formally define analogies and  different approaches to address the two main problems dealing with AR, namely, analogy making and solving. 

Multiple attempts have been made to formalize and manipulate analogies, starting as early as De Saussure's work in 1916 %\cite{linguistique-generale:1916:saussure}, 
but there is not a consensual view on the topic.
Many works rely on the  common view of analogy as a geometric proportion ($a\times d=b\times c$) or as an arithmetic proportion ($ b-a=d-c$), which can be thought of as a {\it parallelogram rule} in a vector space.
This view led to an axiomatic approach \cite{Lepage01} in which analogies are quaternary relations satisfying $3$ postulates: reflexivity,  symmetry, and central permutation.
These postulates imply many other properties: {\it identity} ($a:a::b:b$), {\it internal reversal} (if $a:b::c:d$, then $b:a::d:c$), {\it extreme permutation} (if $a:b::c:d$, then $d:b::c:a$), etc.

Another frequently accepted postulate is {\it uniqueness} stating that if there exists $d$ such that $a:b::c:d$, then $d$ is unique. As argued in \cite{LimPR21}, the latter postulate is debatable and the authors illustrate it through linguistic examples,  {\it e.g.}, $\mathit{wine} : \mathit{French} :: \mathit{beer} : x$ as  $\mathit{Belgian}$, $\mathit{Czech}$ and $\mathit{German}$ seem reasonable. This is further illustrated in \cite{Antic2020} through arithmetic examples. Take, for instance, $20:4::30:x$ that has a clear solution $x=6$. However, $x=9$ is another solution since $(10\cdot 2):2^2::(10\cdot 3):3^2$. Independently, it has also been observed that uniqueness and central permutation are not compatible in simple analogy models in non-Euclidean domains~\cite{murena2018opening}. 

The framework \cite{Antic2020} assumes that pairs $(a,b)$ and $(c,d)$ are interpreted over two, possibly different, algebras $\mathbb{A}=(A,F)$ and  $\mathbb{B}=(B,F)$ with the same functional signature $L$, and called respectively the {\it source} and {\it target domains}. Analogies are then modeled by common rewriting transformations  called {\it justifications} of the form $s\rightarrow t$, where both $s$ and $t$ are $L$-terms in such a way that $a=s^{\mathbb{A}}(\mathbf{e}_1)$ and $b=t^{\mathbb{A}}(\mathbf{e}_1)$, for some $\mathbf{e}_1\in \mathbb{A}^{\mid \mathbf{x} \mid}$, and $c=s^{\mathbb{B}}(\mathbf{e}_2)$ and $d=t^{\mathbb{B}}(\mathbf{e}_2)$, for some $\mathbf{e}_2\in \mathbb{B}^{\mid \mathbf{x} \mid}$. A quadruple $(a,b,c,d)\in A^2\times B^2$ is then said to be an {\it analogical proportion}, denoted $a:b::c:d$, if there are no $a',b'\in A$ and $c', d'\in B$ such that
\begin{enumerate}
    \item the set of justifications of $(a,b,c,d')$  strictly contains that of  $(a,b,c,d)$,
    \item the set of justifications of $(b,a,d,c')$  strictly contains that of  $(b,a,d,c)$, 
    \item the set of justifications of $(c,d,a,b')$ strictly contains that of $(c,d,a,b)$, 
    \item the set of justifications of $(d,c,b,a')$ strictly contains that of  $(d,c,b,a)$.
\end{enumerate} 
Given $\mathbb{A}$ and  $\mathbb{B}$, the relation comprising all analogical proportions will be referred to as a \emph{formal model of analogy} or, simply, as an \emph{analogy}.  Note that every analogy fulfills  internal reversal and extreme permutation. 

This framework accommodates many formal models of analogies, including  the {\it factorial} view of \cite{StroppaY05} and the {\it functional} view of \cite{BarbotMP19,murena2020solving}, except that it is not bound by the central permutation postulate. Such a framework is close to Gentner's symbolic model of analogical reasoning \cite{GentnerH17} based on {\it structure mapping theory} and first implemented in \cite{FalkenhainerFG89}. Both share the view that analogies are compatible with structure preserving maps. 
However, the latter prefers knowledge connected facts to isolated ones, and the former fails to satisfactorily  account for analogies over different conceptual spaces as in $\mathit{wine} : \mathit{French} :: \mathit{beer} : \mathit{Belgian}$.

In this paper, we will focus on the case where the source and target domains coincide (i.e., $\mathbb{B}=\mathbb{A}$), and we will denote the set of all analogies (on $A$) by $\mathcal{A}_A$, and we may omit the subscript $A$ when it is clear from the context. For an analogy $R\in \mathcal{A}_A$ we will adopt the more specific notation $R(a,b,c,d)$
instead of the usual notation $a:b::c:d$, since we will be dealing simultaneously with multiple models of analogy.

\begin{example}\label{ex:Boolean-analogies}
Anti\'c~\cite{Antic-Boolean} has identified the following 5 relations as the formal models of analogies on the two-element set $\{0,1\}$.
Klein's~\cite{Klein} definition of Boolean proportion corresponds to $R_5$, while Miclet and Prade's~\cite{MicletPrade} definition corresponds to $R_4$.
Here and later, we represent a relation as a matrix whose columns are precisely the tuples belonging to the relation.

\begin{align*}
R_1 &:=
\begin{pmatrix}
0 & 1 & 0 & 1 & 1 & 0 & 1 & 0 & 0 & 1 & 0 & 1 \\
0 & 0 & 1 & 1 & 0 & 1 & 0 & 1 & 0 & 0 & 1 & 1 \\
0 & 0 & 0 & 0 & 1 & 1 & 0 & 0 & 1 & 1 & 1 & 1 \\
0 & 0 & 0 & 0 & 0 & 0 & 1 & 1 & 1 & 1 & 1 & 1
\end{pmatrix}
\\
R_2 &:=
\begin{pmatrix}
0 & 1 & 0 & 1 & 1 & 0 & 0 & 1 \\
0 & 0 & 1 & 1 & 0 & 1 & 0 & 1 \\
0 & 0 & 0 & 0 & 1 & 0 & 1 & 1 \\
0 & 0 & 0 & 0 & 0 & 1 & 1 & 1
\end{pmatrix}
\\
R_3 &:=
\begin{pmatrix}
0 & 1 & 1 & 0 & 0 & 1 & 0 & 1 \\
0 & 1 & 0 & 1 & 0 & 0 & 1 & 1 \\
0 & 0 & 1 & 0 & 1 & 1 & 1 & 1 \\
0 & 0 & 0 & 1 & 1 & 1 & 1 & 1
\end{pmatrix}
\\
R_4 &:=
\begin{pmatrix}
0 & 1 & 1 & 0 & 0 & 1 \\
0 & 1 & 0 & 1 & 0 & 1 \\
0 & 0 & 1 & 0 & 1 & 1 \\
0 & 0 & 0 & 1 & 1 & 1
\end{pmatrix}
\\
R_5 &:=
\begin{pmatrix}
0 & 1 & 1 & 0 & 1 & 0 & 0 & 1 \\
0 & 1 & 0 & 1 & 0 & 1 & 0 & 1 \\
0 & 0 & 1 & 1 & 0 & 0 & 1 & 1 \\
0 & 0 & 0 & 0 & 1 & 1 & 1 & 1
\end{pmatrix}
\end{align*}
\end{example}

%\section{Analogy preserving functions /classifiers}

%Miguel writes

%{\flushleft \bf Add notation:} $\mathcal{A}$ is the set of all analogies on $A$
\section{Galois theory for analogical classifiers
}\label{sec4}

%{\flushleft \bf Miguel Writes:} Intro \& survey IJCAI's and SUM papers

As mentioned in the Introduction, analogical inference yields competitive results in classification and recommendation tasks. However, the justification of why and when a classifier is compatible with the   analogical inference principle (AIP) remained rather obscure until the work of Couceiro {\it et al.} \cite{CouceiroHPR17}. In this paper the authors considered the minimal Boolean analogy model (see $R_4$ in Example \ref{ex:Boolean-analogies}) and  addressed the problem of determining those {\it Boolean classifiers for which the AIP always holds}, that is, for which there are no classification errors. Surprisingly, they showed that they correspond to ``analogy preserving'' (see Definition~\ref{def:AP} below) and that they constitute the clone of affine functions.
This result was later generalized to binary classification tasks on nominal (finite) domains in \cite{CouceiroLMPR20} where the authors considered the more stringent notion of ``hard analogy preservation''. By taking the same minimal analogy model made only of analogical proportions of the  form $a:a::b:b$ and $a:b::a:b$  on both the domain and the label set, the authors showed that in this case the sets of hard analogy preserving functions constitute Burle's clones \cite{Burle}.

These preliminary results ask for a better understanding of these analogical classifiers, and in this paper we seek a general theory of analogical classifiers that is not dependent on the underlying sets nor on particular models of analogy. More precisely, we generalize  the existing literature by establishing a Galois theory of analogical classifiers which we then use to explicitly describe the sets of Boolean analogical classifiers with respect to the pairs $(R,S)$ of the known Boolean models $R$ and $S$ of analogy. 
We first recall the notion of analogy preservation and establish some useful results that allow us to use the universal algebraic toolbox.

\begin{definition}
\label{def:AP}
Let $A$ and $B$ be sets, and let $R$ and $S$ be analogical proportions defined on the two sets, respectively.
A function $f \colon A^n \to B$ is \emph{analogy-preserving} (AP for short) relative to $(R,S)$ if for all $\mathbf{a}, \mathbf{b}, \mathbf{c}, \mathbf{d} \in A^n$, the following implication holds:
\[
\bigl(
R(\mathbf{a}, \mathbf{b}, \mathbf{c}, \mathbf{d})
\quad \text{and} \quad
\Ssolvable{S}(f(\mathbf{a}), f(\mathbf{b}), f(\mathbf{c}))
\bigr)
\implies
S(f(\mathbf{a}), f(\mathbf{b}), f(\mathbf{c}), f(\mathbf{d})),
\]
where
$R(\mathbf{a}, \mathbf{b}, \mathbf{c}, \mathbf{d})$ is a shorthand for $(a_i, b_i, c_i, d_i) \in R$ for all $i \in \{1, \dots, n\}$ and
$\Ssolvable{S}(f(\mathbf{a}), f(\mathbf{b}), f(\mathbf{c}))$ means that there is an $x \in B$ with $S(f(\mathbf{a}), f(\mathbf{b}),f(\mathbf{c}), x)$.
Denote by $\mathsf{AP}(R,S)$ the set of all analogy-preserving functions relative to $(R,S)$.
\end{definition}

This relation between functions and formal models of analogy gives rise to a Galois connection whose closed sets of functions correspond exactly to the classes of analogical classifiers that we now describe.

We start by stating and proving some useful results.

\begin{proposition}\label{prop:anpres-pres}
Let $R$ and $S$ be analogical proportions defined on sets $A$ and $B$, respectively.
Then $\mathsf{AP}(R,S) = \Pol (R,S')$, where
\begin{equation}\label{eq:S'}
S' := S \cup \{ (a, b, c, d) \in B^4 \mid \nexists x \in B \colon (a, b, c, x) \in S \}.
\end{equation}
%Consequently, $\mathsf{AP}(R,S)$ is a locally closed minion.
\end{proposition}

\begin{proof}
The condition of Definition~\ref{def:AP} can be written equivalently as follows:
for all $\mathbf{a}_1, \dots, \mathbf{a}_n \in R$, $f(\mathbf{a}_1, \dots, \mathbf{a}_n) \in S'$.
This is exactly what it means that $f$ preserves $(R,S')$.
Therefore, $\mathsf{AP}(R,S) = \Pol (R,S')$.
\end{proof}

\begin{example}
\label{ex:Boolean-analogies-continued}
In continuation to Example~\ref{ex:Boolean-analogies}, the derived relations as in  Proposition~\ref{prop:anpres-pres} corresponding to Anti\'c's analogical proportions on $\{0,1\}$ are the following:
\begin{gather*}
R'_1 = R_1,
\quad
R'_2 = R_2 \cup
\begin{pmatrix}
0 & 0 \\
1 & 1 \\
1 & 1 \\
0 & 1
\end{pmatrix},
\quad
R'_3 = R_3 \cup
\begin{pmatrix}
1 & 1 \\
0 & 0 \\
0 & 0 \\
0 & 1
\end{pmatrix},
\\
R'_4 = R_4 \cup
\begin{pmatrix}
0 & 0 & 1 & 1 \\
1 & 1 & 0 & 0 \\
1 & 1 & 0 & 0 \\
0 & 1 & 0 & 1
\end{pmatrix},
\quad
R'_5 = R_5.
\end{gather*}
\end{example}

To fully describe the sets of the form $\Pol (R,S')$ we need to introduce some variants of the closure conditions discussed in Section~\ref{sec2}

Let $\mathcal{R}$ 
%and $\mathcal{S}$ 
be 
set of $m$-ary relations on $A$. %and $B$, respectively. 
An   $m\times n$ matrix $D$ whose columns belong to a relation $R\in \mathcal{R}$, is called an  \emph{$\mathcal{R}$-locality}.
Let $\mathcal{Q} \subseteq \mathcal{R}_{AB}$, and let $\mathcal{Q}_1 := \{ R \in \mathcal{R}_A \mid \exists S \in \mathcal{R}_B \text{ such that } (R,S) \in \mathcal{Q} \}$. 
A set $\mathcal{C} \subseteq \mathcal{F}_{AB}$ is \emph{
%$(\mathcal{A}_A,\mathcal{A}_B)$
$Q$-locally closed} if for all $f \in \mathcal{F}_{AB}$ (say $f$ is $n$-ary), it holds that $f \in C$ whenever for every $\mathcal{Q}_1$-locality  $D$, either 
\begin{enumerate}
    \item there exists a $g \in \mathcal{C}$ such that $fD = gD$, or 
    
   % $x\in B$ such that  $(fD_{1*},fD_{2*},fD_{3*},x)\in S$. 
%    \item for any relation $R$ in $\mathcal{Q}_1$ such that $D \preceq R$, there is a $T \in \{ S \in \mathcal{R}_B \mid (R,S) \in \mathcal{Q}, \, \mathcal{C} R \subseteq S \}$ such that $f R \notin T$.
    \item for any relation $R$ in $\mathcal{Q}_1$ such that $D \preceq R$ and for any $$T \in \{ S \in \mathcal{R}_B \mid (R,S) \in \mathcal{Q}, \, \mathcal{C} R \subseteq S \}$$ we have that $f R \subseteq T$.
\end{enumerate}

Let $\mathcal{A}'_B := \{ S' \mid S \in \mathcal{A}_B \}$, and let $\mathcal{A}_{AB} :=  \mathcal{A}_A \times \mathcal{A}'_B $. We refer to the elements of $\mathcal{A}_{AB}$ as \emph{analogical constraints} from $A$ to $B$. The set of analogical constraints that are $(C_1,C_2)$-constraints will be denoted by
$$\mathcal{A}_{AB}^{(C_1,C_2)} := \mathcal{A}_{AB} \cap \mathcal{R}_{AB}^{(C_1,C_2)}.$$ 

A set $\mathcal{C}$ is said to be \emph{$(C_1,C_2)$-analogically locally closed} if it is $\mathcal{A}_{AB}^{(C_1,C_2)}$-locally closed.
Note that $\mathcal{A}_{AB} = \mathcal{A}_{AB}^{(\mathcal{J}_A,\mathcal{J}_B)}$, and in this case we simply say that $\mathcal{C}$ is \emph{analogically locally closed.}

%A set $\mathcal{C}$ is said to be \emph{$C$-analogically locally closed} if it is  $\mathcal{A}^{C}_A$-locally closed,  where $\mathcal{A}^{C}_A=\mathcal{A}_A\cap \Inv C$ is the set of all analogies on $A$ invariant under $C$.
% Note that $\mathcal{A}^{C}_A=\mathcal{A}_A$, if $C$ is the clone of projections on $A$, and in this case we simply say that $\mathcal{C}$ is analogically locally closed.

%\begin{proposition}\label{prop:analogy}
%Suppose that $\mathcal{C}$ is definable by some set $\mathcal{Q}$ of analogical constraints. Then, $\mathcal{C}$ is $\mathcal{Q}$-locally closed

%\end{proposition}

%A set $\mathcal{C}\subseteq\mathcal{F}_{AB}$ is \emph{analogically locally closed} if for all $f\in \mathcal{F}_{AB}$ (say $f$ is $n$-ary), it holds that for every $4\times n$ matrix  set whose columns 

\begin{theorem} 
Let  $A$ and $B$ be arbitrary nonempty sets, and let $C_1$ and $C_2$ be clones on $A$ and $B$, respectively.
\begin{enumerate}
     \item A set $\mathcal{C}\subseteq\mathcal{F}_{AB}$  is definable by analogical $(C_1,C_2)$-constraints if and only if it is a $(C_1,C_2)$-analogically locally closed $(C_1,C_2)$-clonoid.
    \item A set $\mathcal{C}\subseteq\mathcal{F}_{AB}$  is definable by analogical constraints if and only if it is an analogically locally closed minion.
\end{enumerate}
\end{theorem}

\begin{proof}
Note that the second statement is a particular case of the first: just take $C_1$ and $C_2$ to be the clones of projections on $A$ and $B$, respectively. We will thus prove the first statement.

To prove that the conditions are necessary, observe that $\mathcal{C}$ is a $(C_1,C_2)$-clonoid by Theorem~\ref{C12clonoids}. It thus remains to show that it is $(C_1,C_2)$-analogically locally closed. Let $f\not\in \mathcal{C}$, say of arity $n$. Hence, there is an analogical $(C_1,C_2)$-constraint $(R,S)\in \mathcal{A}_{AB}^{(C_1,C_2)}$ that is satisfied by every function in $g\in \mathcal{C}$, but not by $f$. Let $\mathbf{a}_1, \dots, \mathbf{a}_n \in R$ such that $f(\mathbf{a}_1, \dots, \mathbf{a}_n) \not\in S$. Consider the $(\mathcal{A}_{AB}^{(C_1,C_2)})_1$-locality $D=(\mathbf{a}_1, \dots, \mathbf{a}_n)$. As every $g\in \mathcal{C}^{(n)}$ satisfies the constraint $(R,S)$, we have that $ fD \neq gD$. Also, $D\preceq R\in (\mathcal{A}_{AB}^{(C_1,C_2)})_1$ and it is clear that $S \in \{ S_0 \in \mathcal{R}_B \mid (R,S_0) \in \mathcal{A}_{AB}^{(C_1,C_2)}, \, \mathcal{C} R \subseteq S_0 \}$, and we have $f R \not\subseteq S$ since $fD\not \in S$.

To prove that the conditions are sufficient, we follow a similar strategy to that in \cite{CouFol09} and  show that for every $n$-ary $f\not \in \mathcal{C}$, there is an analogical $(C_1,C_2)$-constraint $(R,S)\in \mathcal{A}_{AB}^{(C_1,C_2)}$ that is satisfied by every function $g\in \mathcal{C}$, but not by $f$. The set of such analogical $(C_1,C_2)$-constraints will then define $\mathcal{C}$.
%Without loss of generality, we may assume that $\mathcal{C}\neq \varnothing$.

So suppose that $f\not \in \mathcal{C}$. Since $\mathcal{C}$ is $(C_1,C_2)$-analogically locally closed, there is a $(\mathcal{A}_{AB}^{(C_1,C_2)})_1$-locality $D$ such that 
$ fD \neq gD$, for every $n$-ary $g\in \mathcal{C}$, and there exist $R\in (\mathcal{A}_{AB}^{(C_1,C_2)})_1$ with $D\preceq R$, and  $S \in \{ S_0 \in \mathcal{R}_B \mid (R,S_0) \in \mathcal{A}_{AB}^{(C_1,C_2)}, \, \mathcal{C} R \subseteq S_0 \}$ such that  $f R \not \in S$. 

Since $(R,S)$ is an analogical $(C_1,C_2)$-constraint that is satisfied by every function $g\in \mathcal{C}$ (as $\mathcal{C}R\subseteq T$) but not by $f$, this constitutes the desired constraint separating $f$ and $\mathcal{C}$, and the proof is thus complete.
\end{proof}

Dually, a set $\mathcal{Q}$ of analogical $(C_1,C_2)$-constraints is definable by a set $\mathcal{C}\subseteq \mathcal{F}_{AB}$ if $\mathcal{Q} = \Inv \mathcal{C} \cap \mathcal{A}_{AB}^{(C_1,C_2)}$.
The description of the dual closed sets of analogical constraints is then an immediate consequence of Theorem~\ref{thm:4}.

\begin{theorem} 
Let  $A$ and $B$ be arbitrary nonempty sets, and let $C_1$ and $C_2$ be clones on $A$ and $B$, respectively.
\begin{enumerate}
     \item A set $\mathcal{Q}$ of analogical $(C_1,C_2)$-constraints is 
definable by some set 
 $\mathcal{C}\subseteq \mathcal{F}_{AB}$ if and only if there exists a set $\mathcal{Q}_0$ of constraints from $A$ to $B$ that is $(C_1,C_2)$-locally closed and contains the binary equality and 
 the empty constraints, and it is closed under $(C_1,C_2)$-relaxations and conjunctive minors, such that $\mathcal{Q}=\mathcal{A}_{AB}\cap \mathcal{Q}_0$.
    \item A set $\mathcal{Q}$ of analogical constraints from $A$ to $B$ is 
definable by some set 
 $\mathcal{C}\subseteq \mathcal{F}_{AB}$ if and only if there exists a set $\mathcal{Q}_0$ of constraints from $A$ to $B$ that is locally closed and contains the binary equality and 
 the empty constraints, and it is closed under relaxations and conjunctive minors, such that $\mathcal{Q}=\mathcal{A}_{AB}\cap \mathcal{Q}_0$.
\end{enumerate}
\end{theorem}

\section{Explicit description of Boolean analogical classifiers}

%{\flushleft \bf Try:} particular case of model of analogy without central permutation...

%\begin{center}
\begin{table}
\centering
\begin{tabular}{lllllll}
\toprule
\multicolumn{2}{l}{\multirow{2}{*}{$\mathsf{AP}(R,S)$}} & \multicolumn{5}{c}{$S$} \\
&& $R_1$ & $R_2$ & $R_3$ & $R_4$ & $R_5$ \\
\midrule
\multirow{5}{*}{$R$}
& $R_1$ & $\clOmegaOne$ & $\clC$        & $\clC$        & $\clC$ & $\clC$ \\
& $R_2$ & $\clOmegaOne$ & $\clI$        & $\clIneg$     & $\clC$ & $\clC$ \\
& $R_3$ & $\clOmegaOne$ & $\clIneg$     & $\clI$        & $\clC$ & $\clC$ \\
& $R_4$ & $\clL$        & $\clOmegaOne$ & $\clOmegaOne$ & $\clL$ & $\clL$ \\
& $R_5$ & $\clL$        & $\clC$        & $\clC$        & $\clL$ & $\clL$ \\
\bottomrule
\end{tabular}
\medskip
 \caption{Summary of results. Notation: $\clL$ denotes the class of linear (affine) functions, $\clOmegaOne$ the class of projections, negations, and constants, $\clIneg$   the class of negations and constants, $\clI$  the class of projections and constants, and
$\clC$  the class of constants.}
\label{table:summary}
\end{table}

%\bigskip
%\begin{tabular}{ll}
%$\clL$        & linear (affine) functions \\
%$\clOmegaOne$ & projections, negations, and constants \\
%$\clIneg$     & negations and constants \\
%$\clI$        & projections and constants \\
%$\clC$        & constants
%\end{tabular}

%\bigskip
%\end{center}

Recall the formal models of analogy $R_i$, $i\in [5]$, provided by Anti\'c~\cite{Antic-Boolean} on the two-element set $\{0,1\}$ (see Example~\ref{ex:Boolean-analogies}). In this section we make use of the Galois theory described in Section~\ref{sec4} to determine the classes of analogical classifiers $\mathsf{AP}(R_i,R_j)=\Pol(R_i, R_j')$ for all $i, j \in \nset{5}$ (see Example~\ref{ex:Boolean-analogies-continued} for the relations $R'_j$ and Equation \eqref{eq:S'} for the definition of the extension $S'$ of a relation $S$). The results are summarised in Table~\ref{table:summary}, together with our notation for various classes of Boolean functions.

For $a \in \{0,1\}$, let 
$\overline{a} := 1 - a,$ and for $\vect{a} = (a_1, \dots, a_n) \in \{0,1\}^n$, let 
\[
\overline{\vect{a}} := (\overline{a_1}, \dots, \overline{a_n}).
\]
Let $f \colon \{0,1\}^n \to \{0,1\}$.
The \emph{\textup{(}outer\textup{)} negation} $\overline{f}$, the \emph{inner negation} $f^\mathrm{n}$, and the \emph{dual} $f^\mathrm{d}$ of $f$ are the $n$-ary Boolean functions
%  $\overline{f}, f^\mathrm{n}, f^\mathrm{d} \colon \{0,1\}^n \to \{0,1\}$
given by the rules
\[
\overline{f}(\vect{a}) := \overline{f(\vect{a})},
\quad
f^\mathrm{n}(\vect{a}) := f(\overline{\vect{a}}),
\quad
f^\mathrm{d}(\vect{a}) := \overline{f(\overline{\vect{a}})},
\quad
\text{for all $\vect{a} \in \{0,1\}^n$.}
\]
For $C \subseteq \clAll$, we let $\overline{C} := \{ \overline{f} \mid f \in C \}$, $C^\mathrm{n} := \{ f^\mathrm{n} \mid f \in C \}$, $C^\mathrm{d} := \{ f^\mathrm{d} \mid f \in C \}$.

Up to permutation of arguments, the binary Boolean functions are the following:
\begin{itemize}
    \item the constant $0$ and $1$ functions, denoted respectively by
    $0$ and $1$,
    \item  the first projection $\pr_1\colon (x_1,x_2)\mapsto x_1$ and its negation $\neg_1=\overline{\pr_1}$,
    \item the conjunction $\wedge$ and its negation  $\uparrow$,
    \item the disjunction $\vee$ and its negation $\downarrow$,
    \item the implication $\rightarrow$ and its negation $\nrightarrow$, and
    \item the addition $+$ modulo 2 and its negation $\leftrightarrow$.
\end{itemize}
Note that $\uparrow$ and $\downarrow$ are often referred to as {\it Sheffer functions} as each one of them can generate the class of all Boolean functions by taking compositions and variable substitutions.
    
The ternary Boolean functions include the triple sum 
\begin{align*}
\oplus_3\colon (x_1,x_2,x_3)\mapsto (x_1+x_2)+x_3=x_1+(x_2+x_3)
\end{align*}
and its negation $\overline{\oplus_3}=\oplus_3^\mathrm{n}$, and the median function 
\begin{align*}
\mu\colon(x_1,x_2,x_3)&\mapsto (x_1\wedge x_2)\vee (x_1\wedge x_3)\vee(x_2\wedge x_3)\\
&= (x_1\vee x_2)\wedge (x_1\vee x_3)\wedge (x_2\vee x_3)\\
&=\oplus_3(x_1\wedge x_2,x_1\wedge x_3,x_2\wedge x_3)
\end{align*}
and its negation $\overline{\mu}=\mu^\mathrm{n}$.

We also make use of the following terminology.
For a relation $R \subseteq \{0,1\}^m$, we define 
its \emph{negation} $\overline{R}$  by $\overline{R} := \{ \overline{\vect{a}} \mid \vect{a} \in R \}$.

\begin{lemma}\label{lem:dual}
$\Pol(R,S) = \Pol(\overline{R},\overline{S})^\mathrm{d}$.
\end{lemma}

\begin{proof}
We need to show that $f \preserves (R, S)$ if and only if $f^\mathrm{d} \preserves (\overline{R}, \overline{S})$.
Suppose that $f \preserves (R, S)$, and let $M \prec \overline{R}$.
Hence, $\overline{M} \in R$ and we have
$
f^\mathrm{d} M = f^\mathrm{d} \overline{\overline{M}} = \overline{f \overline{M}} \in \overline{S},
$
and thus $f^\mathrm{d} \preserves (\overline{R}, \overline{S})$.
The converse implication follows by the same argument.
\end{proof}

\begin{fact}
\label{lem:inc}
Let $R$ and $S$ be $m$-ary relations on $\{0,1\}$.
\begin{enumerate}[label={\upshape (\roman*)}]
\item\label{lem:inc:id}
$\id \in \Pol(R, S)$ if and only if $R \subseteq S$.
\item\label{lem:inc:neg}
$\neg \in \Pol(R, S)$ if and only if $\overline{R} \subseteq S$.
\item\label{lem:inc:const}
For $a \in \{0,1\}$, if $(a, \dots, a) \in S$, then $a \in \Pol(R,S)$.
\end{enumerate}
\end{fact}

% \begin{proof}
% Immediately obvious.
% \end{proof}

Observe that the constant tuples \vect{0} and \vect{1} belong to every $R_i$ ($i \in \nset{5}$), and thus every such $R_i$ is invariant under $\clI$, i.e., 
$\clI R_i\subseteq R_i$. Hence, for every $i, j \in \nset{5}$, $\Pol(R_i, R'_j)$ is stable under right composition with $\clI$. This leads us to considering  the following notion. 

A function $f$ is said to be a $C$-\emph{minor} of a function $g$ if $f\in gC$. Recall that in the particular case when $C=\mathcal{J}_{\{0,1\}}$, $f$ is  called a {minor} of $g$. The functions $f$ and $g$ are said to be {\it equivalent}, denoted by $f\equiv g$, if $f$ is a minor of $g$ and $g$ is a minor of $f$. For further background on these notions and variants see, e.g., \cite{Lehtonen2006,LehtonenSzendrei,Pippenger}.

Since $\Pol(R_i, R'_j)$ is stable under right composition with $\clI$, this means that if an $\clI$-{minor} $f$ of a function $g$ does not belong to $\Pol(R_i, R'_j)$, then neither does $g$.
This observation will be used repeatedly in the proofs of the results that will follow.

\begin{lemma}
\label{lem:Pol-easy}
For any $i, j \in \nset{5}$, the following statements hold.
\begin{enumerate}[label={\upshape (\roman*)}]
\item\label{lem:Pol-easy:plus}
If $\id \notin \Pol(R_i, R'_j)$ or $\neg \notin \Pol(R_i, R'_j)$, then $\mathord{+}, \mathord{\leftrightarrow}, \mathord{\rightarrow}, \mathord{\nrightarrow} \notin \Pol(R_i, R'_j)$.
\item\label{lem:Pol-easy:triplesum}
If $\mathord{+} \notin \Pol(R_i, R'_j)$, then $\mathord{\oplus_3} \notin \Pol(R_i, R'_j)$.
\item\label{lem:Pol-easy:mu}
If $\mathord{\wedge} \notin \Pol(R_i, R'_j)$ or $\mathord{\vee} \notin \Pol(R_i, R'_j)$, then $\mu \notin \Pol(R_i, R'_j)$.
\end{enumerate}
\end{lemma}

\begin{proof} These results follow immediately from the observation that the sets of the form $\Pol(R_i, R'_j)$ are closed under taking  $\clI$-{minors}. The proof is then obtained by verifying that the functions in the antecedent of the implications are $\clI$-{minors} of those in the consequent of the implication. 
We prove this claim by explicitly for \ref{lem:Pol-easy:plus}, and leave the remaining for the reader.

\ref{lem:Pol-easy:plus} For every $x\in \{0,1\}$,
%If $\id \notin \Pol(R_i, R'_j)$, then there exists an $\vect{a} \in R_i$ such that $\id(\vect{a}) \notin R'_j$.
%Since $\vect{0}, \vect{1} \in R_i$, we also have
$\id(x)=\mathord{+}(x, 0) = \mathord{\leftrightarrow}(x,1)=\mathord{\rightarrow}(1,x)=\mathord{\nrightarrow}(x,0).$
Similarly, we also have $\neg(x)=\mathord{+}(x, 1) =\mathord{\leftrightarrow}(x, 0) = \neg(x)=\mathord{\rightarrow}(x, 0) =\mathord{\nrightarrow}(1, x)$.

\ref{lem:Pol-easy:triplesum} and 
%If $\mathord{+} \notin \Pol(R_i, R'_j)$, then there exist $\vect{a}, \vect{b} \in R_i$ such that $\mathord{+}(\vect{a}, \vect{b}) \notin R'_j$.
%Since $\vect{0} \in R_i$, we also have $\mathord{\oplus_3}(\vect{a}, \vect{b}, \vect{0}) = \mathord{+}(\vect{a}, \vect{b}) \notin R'_j$.
\ref{lem:Pol-easy:mu} are shown similarly.
%The claim follows immediately from the fact that $\mu(\vect{a}, \vect{b}, \vect{1}) = \mathord{\wedge}(\vect{a}, \vect{b})$ and $\mu(\vect{a}, \vect{b}, \vect{0}) = \mathord{\vee}(\vect{a}, \vect{b})$.
% Making use of the fact that $\mu(\vect{a}, \vect{b}, \vect{1}) = \mathord{\wedge}(\vect{a}, \vect{b})$ and $\mu(\vect{a}, \vect{b}, \vect{0}) = \mathord{\vee}(\vect{a}, \vect{b})$, it follows immediately that if $\mathord{\wedge} \notin \Pol(R_i, R'_j)$ or $\mathord{\vee} \notin \Pol(R_i, R'_j)$, then $\mu \notin \Pol(R_i, R'_j)$.
% If $\mathord{\wedge} \notin \Pol(R_i, R'_j)$, then there exist $\vect{a}, \vect{b} \in R_i$ such that $\mathord{\wedge}(\vect{a}, \vect{b}) \notin R'_j$.
% Since $\vect{0} \in R_i$, we also have $\mu(\vect{a}, \vect{b}, \vect{1}) = \mathord{\wedge}(\vect{a}, \vect{b}) \notin R'_j$.
% The proof is similar in the case when $\mathord{\vee} \notin \Pol(R_i, R'_j)$; we just make use of the fact that $\mu(\vect{a}, \vect{b}, \vect{0}) = \mathord{\vee}(\vect{a}, \vect{b})$.
\end{proof}

\begin{lemma}\label{lem:C}
For any $i, j \in \nset{5}$, if $\id, \neg \notin \Pol(R_i,R'_j)$, then $\Pol(R_i,R'_j) = \clC$.
\end{lemma}

\begin{proof}
Just note that the only functions that do not have $\clI$-{minors} in $\{\id, \neg\}$ are the constant functions. The result thus follows from the observations above.
%We have $0, 1 \in \Pol(R_i,R'_j)$ by Fact~\ref{lem:inc}\ref{lem:inc:const}.
%Since $\id, \neg \notin \Pol(R_i,R'_j)$, any function whose diagonal $\id$ or $\neg$ is excluded from $\Pol(R_i,R'_j)$.
%Up to equivalence, $\mathord{+}$, $\mathord{\leftrightarrow}$, $\mathord{\nrightarrow}$, and $\mathord{\rightarrow}$ are the only binary functions whose diagonal is $0$ or $1$.
%It follows from Lemma~\ref{lem:Pol-easy}\ref{lem:Pol-easy:plus} that $\mathord{+}, \mathord{\leftrightarrow}, \mathord{\nrightarrow}, \mathord{\rightarrow} \notin \Pol(R_i,R'_j)$.
%We are left with $\Pol(R_i,R'_j) = \clC$.
\end{proof}

\begin{proposition}
$\mathsf{AP}(R_1,R_1)=\Pol(R_1,R'_1) = \clOmegaOne$.
\end{proposition}

\begin{proof}
Since $R_1 = R'_1$, it follows immediately that $\Pol(R_1,R'_1) = \Pol R_1$ is a clone.
\begin{itemize}
\item
$0, 1, \id, \neg \in \Pol(R_1,R'_1)$ by Fact~\ref{lem:inc}.

\item
$\mathord{\wedge}, \mathord{\vee}, \mathord{\oplus_3} \notin \Pol(R_1,R'_1)$ because
\begin{align*}
&
\mathord{\wedge}
\begin{pmatrix} 1 & 0 \\ 0 & 1 \\ 1 & 1 \\ 0 & 0 \end{pmatrix}
= \begin{pmatrix} 0 \\ 0 \\ 1 \\ 0 \end{pmatrix}
\notin R'_1,
&&
\mathord{\vee}
\begin{pmatrix} 1 & 0 \\ 0 & 1 \\ 1 & 1 \\ 0 & 0 \end{pmatrix}
= \begin{pmatrix} 1 \\ 1 \\ 1 \\ 0 \end{pmatrix}
\notin R'_1,
\\ &
\mathord{\oplus_3}
\begin{pmatrix} 1 & 0 & 1 \\ 0 & 1 & 1 \\ 1 & 1 & 1 \\ 0 & 1 & 1 \end{pmatrix}
= \begin{pmatrix} 0 \\ 0 \\ 1 \\ 0 \end{pmatrix}
\notin R'_1.
\end{align*}
\item
$\mu \notin \Pol(R_1,R'_1)$ follows by Lemma~\ref{lem:Pol-easy}\ref{lem:Pol-easy:mu}.
\end{itemize}
It follows that $\Pol R_1$ does not contain generators of any minimal clone with essentially at least binary functions.
From this it follows that $\Pol(R_1,R'_1) = \clgen{0,1,\neg} = \clOmegaOne$.
\end{proof}

\begin{proposition}\label{R2R2}
$\mathsf{AP}(R_2,R_2)=\Pol(R_2,R'_2) = \clI$.
\end{proposition}

\begin{proof}
Observe that
\begin{itemize}
\item $0, 1, \id \in \Pol(R_2,R'_2)$ by Fact~\ref{lem:inc}.
\item $\neg, \mathord{\wedge}, \mathord{\vee} \notin \Pol(R_2,R'_2)$ because
\begin{align*}
&
\neg
\begin{pmatrix} 0 \\ 1 \\ 0 \\ 0 \end{pmatrix}
=
\begin{pmatrix} 1 \\ 0 \\ 1 \\ 1 \end{pmatrix}
\notin R'_2,
&&
\mathord{\wedge}
\begin{pmatrix} 1 & 0 \\ 0 & 0 \\ 1 & 1 \\ 0 & 1 \end{pmatrix}
= \begin{pmatrix} 0 \\ 0 \\ 1 \\ 0 \end{pmatrix}
\notin R'_2,
&&
\mathord{\vee}
\begin{pmatrix} 1 & 0 \\ 0 & 0 \\ 1 & 1 \\ 0 & 1 \end{pmatrix}
= \begin{pmatrix} 1 \\ 0 \\ 1 \\ 1 \end{pmatrix}
\notin R'_2.
\end{align*}
\item $ \mathord{\leftrightarrow} \notin \Pol(R_2,R'_2)$ follows by Lemma~\ref{lem:Pol-easy}\ref{lem:Pol-easy:plus}.
%$\mathord{\oplus_3} \notin \Pol(R_2,R'_2)$ follows by Lemma~\ref{lem:Pol-easy}\ref{lem:Pol-easy:triplesum}.\marginpar{Add remaining, e.g., $\overline{\oplus_3}$}
\end{itemize}
The remaining functions that are neither projections nor constants have $\clI$-{minors} in $\{\mathord{\neg}, \mathord{\wedge},\mathord{\vee}, \mathord{\leftrightarrow}\}$.
From this it follows that $\Pol(R_2,R'_2) = \clI$.
\end{proof}

\begin{proposition}
$\mathsf{AP}(R_3,R_3)=\Pol(R_3,R'_3) = \clI$.
\end{proposition}

\begin{proof}
Since $R_3 = \overline{R_2}$, $R'_3 = \overline{R'_2}$, and $\clI^\mathrm{d} = \clI$, the claim follows from Lemma~\ref{lem:dual} and Proposition~\ref{R2R2}.
\end{proof}

\begin{proposition}\label{R5R5}
$\mathsf{AP}(R_5,R_5)=\Pol(R_5,R'_5) = \clL$.
\end{proposition}

\begin{proof}
Since $R_5 = R'_5$, it follows immediately that $\Pol(R_5,R'_5) = \Pol R_5$ is a clone, and it is well known that $\Pol R_5 = \clL$.
%Note that this result $$ was first obtained 
\end{proof}

\begin{proposition}
\leavevmode\newline
$\mathsf{AP}(R_4,R_4)= \mathsf{AP}(R_4,R_5)=\mathsf{AP}(R_5,R_4)= \mathsf{AP}(R_4,R_1) =\mathsf{AP}(R_5,R_1)= \clL$.
\end{proposition}

\begin{proof}
We make use of the fact that $\mathsf{AP}(R,S)=\Pol(R,S')$, and prove the equality $\mathsf{AP}(R_4,R_4)=\Pol(R_4,R'_4) = \clL$.
The equalities $\Pol(R_4,R'_5) = \Pol(R_5,R'_4) = \clL$ are established by the same argument.
\begin{itemize}
\item Since $(R_4,R'_4)$ is a relaxation of $(R_5,R'_5)$,  $\Pol(R_4,R'_4) \supseteq \Pol(R_5,R'_5) = \clL$ by Proposition~\ref{R5R5}.

\item
$\mathord{\wedge}, \mathord{\vee} \notin \Pol(R_4,R'_4)$ because
\begin{align*}
&
\mathord{\wedge}
\begin{pmatrix} 1 & 0 \\ 0 & 0 \\ 1 & 1 \\ 0 & 1 \end{pmatrix}
= \begin{pmatrix} 0 \\ 0 \\ 1 \\ 0 \end{pmatrix}
\notin R'_4,
&&
\mathord{\vee}
\begin{pmatrix} 1 & 0 \\ 1 & 1 \\ 0 & 0 \\ 0 & 1 \end{pmatrix}
= \begin{pmatrix} 1 \\ 1 \\ 0 \\ 1 \end{pmatrix}
\notin R'_4.
\end{align*}

%\item
%$\mu \notin \Pol(R_4,R'_4)$ follows by Lemma~\ref{lem:Pol-easy}\ref{lem:Pol-easy:mu}.

\item
$\mathord{\uparrow}, \mathord{\downarrow} \notin \Pol(R_4,R'_4)$ because
\begin{align*}
&
\mathord{\uparrow}
\begin{pmatrix} 0 & 1 \\ 0 & 0 \\ 1 & 1 \\ 1 & 0 \end{pmatrix}
= \begin{pmatrix} 1 \\ 1 \\ 0 \\ 1 \end{pmatrix}
\notin R'_4,
&&
\mathord{\downarrow}
\begin{pmatrix} 0 & 1 \\ 1 & 1 \\ 0 & 0 \\ 1 & 0 \end{pmatrix}
= \begin{pmatrix} 0 \\ 0 \\ 1 \\ 0 \end{pmatrix}
\notin R'_4.
\end{align*}

\item
$\mathord{\nrightarrow}, \mathord{\rightarrow} \notin \Pol(R_4,R'_4)$ because
\begin{align*}
&
\mathord{\nrightarrow}
\begin{pmatrix} 0 & 1 \\ 0 & 0 \\ 1 & 1 \\ 1 & 0 \end{pmatrix}
= \begin{pmatrix} 0 \\ 0 \\ 0 \\ 1 \end{pmatrix}
\notin R'_4,
&&
\mathord{\rightarrow}
\begin{pmatrix} 0 & 1 \\ 0 & 0 \\ 1 & 1 \\ 1 & 0 \end{pmatrix}
= \begin{pmatrix} 1 \\ 1 \\ 1 \\ 0 \end{pmatrix}
\notin R'_4.
\end{align*}
\end{itemize}

The result now follows by observing that any function outside of $\clL$ has an $\clI$-minor in $\{ \mathord{\wedge}, \mathord{\vee}, \mathord{\uparrow}, \mathord{\downarrow}, \mathord{\nrightarrow}, \mathord{\rightarrow} \}$. 
To see that this claim holds, we make use of the well-known fact that every Boolean function can be represented by a unique multilinear polynomial over the 2-element field $GF(2)$.

So suppose that $f\not \in \clL$, i.e., there is a monomial in the multilinear representation of $f$ with degree\footnote{The degree of a monomial is the number of variables in it, and the degree of a multilinear polynomial is the largest degree among its monomials.} at least 2. Take a minimal monomial with at least 2 variables; without loss of generality, suppose that it is $x_1x_2x_3\cdots x_j$ . 

Consider the $\clI$-minor $f'$ of $f$ obtained by:
\begin{itemize}
    \item substituting $1$ for each $x_3,\cdots , x_j$, and 
    \item substituting $0$ for every other variable in the polynomial representation of $f$.
\end{itemize}
It is not difficult to verify that $f'=x_1x_2+ax_1+bx_2+c$, for some $a,b,c\in \{0,1\}$, and 
\begin{itemize}
    \item $f' \equiv  \mathord{\wedge}$, if $a=b=c=0$;
    \item $f' \equiv  \mathord{\uparrow}$, if $a=b=0$ and $c=1$;
     \item $f' \equiv   \mathord{\vee}$, if $a=b=1$ and $c=0$;
      \item $f'\equiv  \mathord{\downarrow}$, if $a=b=c=1$;
       \item $f'\equiv  \mathord{\nrightarrow}$, if $a=c=1$ and $b=0$, or if $b=c=1$ and $a=0$;
       \item $f'\equiv  \mathord{\rightarrow}$, if $a=1$ and $b=c=0$, or if $b=1$ and $a=c=0$.
\end{itemize}
This completes the proof of the claim and, hence, of the proposition.
\end{proof}

\begin{proposition}
$\mathsf{AP}(R_1,R_2) = \mathsf{AP}(R_1,R_3) = \clC$.
\end{proposition}

\begin{proof}
We have $0, 1 \in \Pol(R_1,R'_2)$ and $\id, \neg \notin \Pol(R_1,R'_2)$ by Fact~\ref{lem:inc}. Therefore, $\Pol(R_1,R'_2) = \clC$ follows by Lemma~\ref{lem:C}.
Since $R_1 = \overline{R_1}$, $R'_3 = \overline{R'_2}$, and $\clC^\mathrm{d} = \clC$, the equality $\Pol(R_1,R'_3) = \clC$ follows by Lemma~\ref{lem:dual}.
\end{proof}

\begin{proposition}
$\mathsf{AP}(R_2,R_1) = \mathsf{AP}(R_3,R_1) = \clOmegaOne$.
\end{proposition}

\begin{proof}
We only prove the equality $\Pol(R_2,R'_1) = \clOmegaOne$.
Since $R_3 = \overline{R_2}$, $R'_1 = \overline{R'_1}$, and $\clOmegaOne^\mathrm{d} = \clOmegaOne$, the equality $\Pol(R_3,R'_1) = \clOmegaOne$ follows by Lemma~\ref{lem:dual}.

\begin{itemize}
\item $0, 1, \id, \neg \in \Pol(R_2,R'_1)$ by Fact~\ref{lem:inc}.
\item $\mathord{+}, \mathord{\nrightarrow}, \mathord{\wedge}, \mathord{\vee} \notin \Pol(R_2,R'_1)$ because
\begin{align*}
&
\mathord{+}
\begin{pmatrix} 1 & 1 \\ 0 & 0 \\ 0 & 1 \\ 0 & 0 \end{pmatrix}
= \begin{pmatrix} 0 \\ 0 \\ 1 \\ 0 \end{pmatrix}
\notin R'_1,
&&
\mathord{\nrightarrow}
\begin{pmatrix} 1 & 1 \\ 0 & 1 \\ 1 & 0 \\ 0 & 0 \end{pmatrix}
= \begin{pmatrix} 0 \\ 0 \\ 1 \\ 0 \end{pmatrix}
\notin R'_1,
\\ &
\mathord{\wedge}
\begin{pmatrix} 1 & 0 \\ 0 & 0 \\ 1 & 1 \\ 0 & 1 \end{pmatrix}
= \begin{pmatrix} 0 \\ 0 \\ 1 \\ 0 \end{pmatrix}
\notin R'_1,
&&
\mathord{\vee}
\begin{pmatrix} 1 & 1 \\ 1 & 0 \\ 0 & 1 \\ 0 & 0 \end{pmatrix}
= \begin{pmatrix} 1 \\ 1 \\ 1 \\ 0 \end{pmatrix}
\notin R'_1.
\end{align*}

%\item $\mathord{\oplus_3}, \mu \notin \Pol(R_2,R'_1)$ follows by Lemma~\ref{lem:Pol-easy}\ref{lem:Pol-easy:triplesum}, \ref{lem:Pol-easy:mu}.

\item Since $R'_1 = \overline{R'_1}$, it follows immediately from the above that $\mathord{\leftrightarrow}, \mathord{\rightarrow}, \mathord{\uparrow}, \mathord{\downarrow} \notin \Pol(R_2,R'_1)$.
\end{itemize}
From these observations and using the technique based on $\clI$-minors, it follows that $\Pol(R_2,R'_1) = \clOmegaOne$.
\end{proof}

\begin{proposition}
$\mathsf{AP}(R_2,R_3) = \mathsf{AP}(R_3,R_2) = \clIneg$.
\end{proposition}

\begin{proof}
We only prove the equality $\Pol(R_2,R'_3) = \clIneg$.
Since $R_3 = \overline{R_2}$, $R'_2 = \overline{R'_3}$, and $\clIneg^\mathrm{d} = \clIneg$, the equality $\Pol(R_3,R'_2) = \clIneg$ follows by Lemma~\ref{lem:dual}.

\begin{itemize}
\item $0, 1, \neg \in \Pol(R_2,R'_3)$ and $\id \notin \Pol(R_2,R'_3)$ by Fact~\ref{lem:inc}.
\item $\mathord{+}, \mathord{\leftrightarrow}, \mathord{\nrightarrow}, \mathord{\rightarrow} \notin \Pol(R_2,R'_3)$ follows from Lemma~\ref{lem:Pol-easy}\ref{lem:Pol-easy:plus}.
%\item $\mathord{\oplus_3} \notin \Pol(R_2,R'_3)$ follows from Lemma~\ref{lem:Pol-easy}\ref{lem:Pol-easy:triplesum}.
\end{itemize}
Any other function has an $\clI$-minor in $\{\mathord{+}, \mathord{\leftrightarrow}, \mathord{\nrightarrow}, \mathord{\rightarrow}\}$. Thus $\Pol(R_2,R'_3) = \clIneg$.
\end{proof}

\begin{proposition}
$\mathsf{AP}(R_4,R_2) = \mathsf{AP}(R_4,R_3) = \clOmegaOne$.
\end{proposition}

\begin{proof}
We only prove the equality $\Pol(R_4,R'_2) = \clOmegaOne$.
Since $R_4 = \overline{R_4}$, $R'_3 = \overline{R'_2}$, and $(\clOmegaOne)^\mathrm{d} = \clOmegaOne$, the equality $\Pol(R_4,R'_3) = \clOmegaOne$ follows by Lemma~\ref{lem:dual}.
\begin{itemize}
\item $0, 1, \id, \neg \in \Pol(R_4,R'_2)$ by Fact~\ref{lem:inc}.
\item $\mathord{+}, \mathord{\leftrightarrow}, \mathord{\rightarrow}, \mathord{\nrightarrow}, \mathord{\wedge}, \mathord{\vee}, \mathord{\uparrow}, \mathord{\downarrow} \notin \Pol(R_4,R'_2)$ because
\begin{align*}
&
\mathord{+}
\begin{pmatrix} 1 & 0 \\ 0 & 0 \\ 1 & 1 \\ 0 & 1 \end{pmatrix}
= \begin{pmatrix} 1 \\ 0 \\ 0 \\ 1 \end{pmatrix}
\notin R'_2,
&&
\mathord{\leftrightarrow}
\begin{pmatrix} 0 & 0 \\ 1 & 0 \\ 0 & 1 \\ 1 & 1 \end{pmatrix}
= \begin{pmatrix} 1 \\ 0 \\ 0 \\ 1 \end{pmatrix}
\notin R'_2,
\\ &
\mathord{\rightarrow}
\begin{pmatrix} 0 & 0 \\ 0 & 1 \\ 1 & 0 \\ 1 & 1 \end{pmatrix}
= \begin{pmatrix} 1 \\ 1 \\ 0 \\ 1 \end{pmatrix}
\notin R'_2,
&&
\mathord{\nrightarrow}
\begin{pmatrix} 0 & 0 \\ 0 & 1 \\ 1 & 0 \\ 1 & 1 \end{pmatrix}
= \begin{pmatrix} 0 \\ 0 \\ 1 \\ 0 \end{pmatrix}
\notin R'_2,
\\ &
\mathord{\wedge}
\begin{pmatrix} 0 & 0 \\ 0 & 1 \\ 1 & 0 \\ 1 & 1 \end{pmatrix}
= \begin{pmatrix} 0 \\ 0 \\ 0 \\ 1 \end{pmatrix}
\notin R'_2,
&&
\mathord{\vee}
\begin{pmatrix} 1 & 1 \\ 0 & 1 \\ 1 & 0 \\ 0 & 0 \end{pmatrix}
= \begin{pmatrix} 1 \\ 1 \\ 1 \\ 0 \end{pmatrix}
\notin R'_2,
\\ &
\mathord{\uparrow}
\begin{pmatrix} 0 & 0 \\ 0 & 1 \\ 1 & 0 \\ 1 & 1 \end{pmatrix}
= \begin{pmatrix} 1 \\ 1 \\ 1 \\ 0 \end{pmatrix}
\notin R'_2,
&&
\mathord{\downarrow}
\begin{pmatrix} 1 & 1 \\ 0 & 1 \\ 1 & 0 \\ 0 & 0 \end{pmatrix}
= \begin{pmatrix} 0 \\ 0 \\ 0 \\ 1 \end{pmatrix}
\notin R'_2.
\end{align*}
%\item $\mathord{\oplus_3}, \mu \notin \Pol(R_2,R'_3)$ follows from Lemma~\ref{lem:Pol-easy}\ref{lem:Pol-easy:triplesum}, \ref{lem:Pol-easy:mu}.
\end{itemize}
Since any other function has an $\clI$-minor in $\{\mathord{+}, \mathord{\leftrightarrow}, \mathord{\rightarrow}, \mathord{\nrightarrow}, \mathord{\wedge}, \mathord{\vee}, \mathord{\uparrow}, \mathord{\downarrow}\}$, it follows that $\Pol(R_2,R'_3) = \clOmegaOne$.
\end{proof}

\begin{proposition}
For any
\[
(R,S) \in \{ (R_1,R_4), (R_1,R_5), (R_2,R_4), (R_2,R_5), (R_3,R_5), (R_5,R_2), (R_5,R_3) \},
\]
we have $\mathsf{AP}(R,S) = \clC$.
% $\Pol(R_1,R'_4) = \Pol(R_1,R'_5) = \Pol(R_2,R'_4) = \Pol(R_2,R'_5) = \Pol(R_3,R'_5) = \Pol(R_5,R'_2) = \Pol(R_5,R'_3) = \clC$.
\end{proposition}

\begin{proof}
We have $0, 1 \in \Pol(R,S)$ and $\id, \neg \notin \mathsf{AP}(R,S)$ by Fact~\ref{lem:inc}, so $\mathsf{AP}(R,S) = \clC$ follows by Lemma~\ref{lem:C}.
\end{proof}

\section{Conclusion and perspectives}\label{sec12}

%Discussions should be brief and focused. In some disciplines use of Discussion or `Conclusion' is interchangeable. It is not mandatory to use both. Some journals prefer a section `Results and Discussion' followed by a section `Conclusion'. Please refer to Journal-level guidance for any specific requirements. 

In this paper we addressed and tackled the question of determining classes of analogical classifiers. Our approach makes use of model theoretic and universal algebraic tools that we used to establish a general Galois theory of such analogical classifiers that does not depend on the underlying domains nor the formal models of analogy considered.
In the particular case of Boolean analogies, in which several formal models of analogy have been identified, we made use of this Galois framework to explicitly describe, for each pair of known analogies, the respective sets of analogical classifiers. 

As future work, we intend to further explore different formal models of analogy that may be obtained by considering different algebraic signatures. For instance, in \cite{Antic-Boolean}, Anti\'c only considered reducts of the 2-element Boolean algebra for which he provided all Boolean analogies. However, other algebraic signatures and their reducts could be considered, such as that of median algebras and fields. These may give rise to further models of analogy considered over different underlying domains, and for which we may obtain other sets of analogical classifiers. Further related questions of model theoretic and universal algebraic flavour will also be addressed.

%\section{Conclusion}\label{sec13}

%Conclusions may be used to restate your hypothesis or research question, restate your major findings, explain the relevance and the added value of your work, highlight any limitations of your study, describe future directions for research and recommendations. 

%In some disciplines use of Discussion or 'Conclusion' is interchangeable. It is not mandatory to use both. Please refer to Journal-level guidance for any specific requirements. 

%\backmatter

\subsection*{Acknowledgments}
The authors wish to thank Esteban Marquer and Pierre-Alexandre Murena for the insightful discussions and useful suggestions for improving this manuscript.

This research was partially supported by TAILOR, a project funded by EU Horizon 2020 research and innovation program under GA No 952215, and the Inria Project Lab ``Hybrid Approaches for Interpretable AI'' (HyAIAI).

This work was also partially funded by National Funds through the FCT -- Funda\c{c}\~{a}o para a Ci\^{e}ncia e a Tecnologia, I.P., under the scope of the project UIDB/00297/2020 (Center for Mathematics and Applications) and the project PTDC/MAT-PUR/31174/2017.

%\bibliographystyle{unsrtnat}
\bibliographystyle{plain}
 \bibliography{bib.bib}
%\bibliography{bib}% common bib file


\end{document}